\documentclass[10pt,twocolumn,letterpaper]{article}
\usepackage{subcaption}
\usepackage{cvpr}
\usepackage{times}
\usepackage{epsfig}
\usepackage{graphicx}
\usepackage{amsmath}
\usepackage{amssymb}
\usepackage[numbers,sort&compress]{natbib}
\usepackage{listings}
\usepackage{color}

\definecolor{dkgreen}{rgb}{0,0.6,0}
\definecolor{gray}{rgb}{0.5,0.5,0.5}
\definecolor{mauve}{rgb}{0.58,0,0.82}

\makeatletter
\@namedef{ver@everyshi.sty}{}
\makeatother
\usepackage{tikz}

\usepackage[
    pagebackref=true,
    breaklinks=true,
    letterpaper=true,
    colorlinks,
    bookmarks=false]{hyperref}

\cvprfinalcopy % *** Uncomment this line for the final submission

\def\cvprPaperID{9} % *** Enter the CVPR Paper ID here

\ifcvprfinal\fi

\usepackage{dsfont}
\usepackage[utf8]{inputenc}
\usepackage[english]{babel}
\usepackage{amsthm}
\usepackage{xcolor}
\usepackage{cancel}
\usepackage{wasysym}
\newtheorem{definition}{Definition}

\newtheorem{example}{Example}

\newtheorem{proposition}{Proposition}

\newcommand{\paren}[1]{\left( #1 \right)}

\newcommand{\set}[1]{\left\{ #1 \right\}}

\def\final{}

\ifdefined\final
\else
\def\enablecomments{}
\fi

\ifdefined\draft
\def\enablecomments{}
\fi

\usepackage{soul}
\usepackage[normalem]{ulem}

\definecolor{LightGreen}{rgb}{0.80,1.00,0.80}
\definecolor{LightBlue}{rgb}{0.80,0.80,1.00}
\definecolor{LightRed}{rgb}{1.00,0.80,0.80}
\definecolor{LightPurple}{rgb}{1.00,0.80,1.00}
\definecolor{LightGray}{rgb}{0.90,0.90,0.90}

\soulregister{\method}{7}
\soulregister{\xspace}{7}
\soulregister{\emph}{7}
\soulregister{\cite}{7}

\ifdefined\enablecomments
  \DeclareRobustCommand{\commentformat}[3]{\sethlcolor{#2}\textsf{\hl{#1: #3}}}
  \newcommand{\sm}    [1]{{\scriptsize\sethlcolor{LightGray}\hl{\textsf{#1}}}}
\else
  \newcommand{\commentformat}[3]{}
  \newcommand{\sm}    [1]{}
\fi

\newcommand{\pxm}   [1]{\commentformat{PM}{LightGreen}{#1}}

\newcommand{\zifan} [1]{\commentformat{ZW}{LightBlue}{#1}}

\begin{document}

%%%%%%%%% TITLE
\title{Interpreting Interpretations: Organizing Attribution Methods by Criteria}

\author{Zifan Wang, Piotr Mardziel, Anupam Datta, Matt Fredrikson\\
Carnegie Mellon Univeristy\\
Moffett Field, CA, 94089\\
{\tt\small zifanw@andrew.cmu.edu}
% \and
}

\maketitle
%\thispagestyle{empty}
%%%%%%%%% ABSTRACT\
\begin{abstract}
% problem
%  To date, dozens of attribution methods and variants have been proposed, each 
  Motivated by distinct, though related, criteria, a growing number of attribution methods have been developed to interprete deep learning. While each relies on the interpretability of the concept of ``importance`` and our ability to visualize patterns, explanations produced by the methods often differ.
  As a result, input attribution for vision models fail to provide  any level of human understanding of model behaviour. In this work we expand the foundations of human-understandable concepts with which attributions can be interpreted beyond "importance" and its visualization; we incorporate the logical concepts of necessity and sufficiency, and the concept of proportionality. We define metrics to represent these concepts as quantitative aspects of an attribution. This allows us to compare attributions produced by different methods and interpret them in novel ways: to what extent does this attribution (or this method) represent the necessity or sufficiency of the highlighted inputs, and to what extent is it proportional? We evaluate our measures on a collection of methods explaining convolutional neural networks (CNN) for image classification. We conclude that some attribution methods are more appropriate for interpretation in terms of necessity while others are in terms of sufficiency, while no method is always the most appropriate in terms of both. 
  
%   With the general results and analyses for individual test cases, a practitioner can incorporate the concepts of necessity, sufficiency, and proportionality into the their interpretation of attribution and thus attain a new higher level of understanding of model behaviour.

  % We define quantitatively and evaluate attribution
  % methods along necessity and sufficiency ordering and proportionality measures, demonstrating
  % the relationships between the methods and optimal attributions for each criterion. \\
  
\end{abstract}

%%%%%%%%% BODY TEXT

\section{Introduction}
% \footnote{Under Submission}
% While deep learning has become the state of the art in
% numerous machine learning applications, deep models are especially resistant to
% human understanding. The gap between power and interpretability is growing especially wide in
% vision and audio applications where adversarial examples \cite{45816} demonstrate that models incorporate
% semantically meaningless features (road stop signs with human-imperceptible changes being
% classified as speed limit signs \cite{eykholt2017robust}) .
Among
approaches for interpreting opaque models are input attribution which assign to each
model input a level of contribution to its output. When visualized alongside inputs,
an attribution gives a human interpreter some notion of what about the input is important to the
outcome (see, for example, Figure~\ref{fig:example_duck}). Being explanations of highly complex systems intended for highly complex humans, attributions have been varied in their approaches and sometimes produce distinct explanations
for the same outputs.

  \begin{figure}[h!]
  \begin{center}
    % \fbox{\rule{0pt}{2in} \rule{.9\linewidth}{0pt}}
    \includegraphics[width=\linewidth]{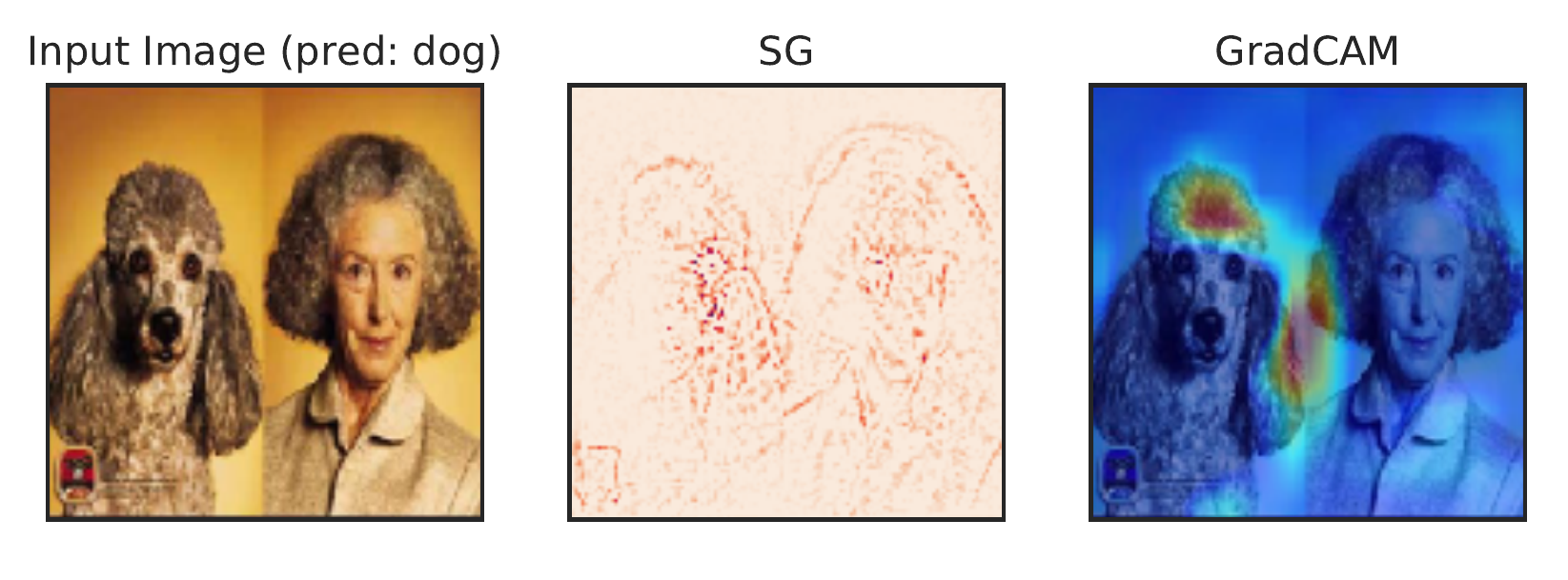}
  \end{center}
  \caption{Interpreting the neural network prediction with different method may come with divergence. Left: the input with both predicted class and groundtruth being \texttt{dog}. Middle: SmoothGrad \cite{smilkov2017smoothgrad}. Pixels with deeper color have higher attribution scores. Right: GradCAM\cite{selvaraju2016gradcam}. Regions with more heat localize the more relevant spatial locations. Questions: \textit{Is the model using the lady to classify the dog? Which interpretation is supposed to be used?} }
  \label{fig:lady_or_dog}
\end{figure}

Nevertheless, save for the earliest approaches, attribution methods distinguish themselves with
one or more desirable criteria. Scaling criteria such as \textit{completeness}~\cite{sundararajan2017axiomatic}, \textit{sensitivity-n}~\cite{ancona2017better}, 
\textit{linear-agreement}~\cite{leino2018influencedirected, sundararajan2017axiomatic} calibrate attribution to the change in output as
compared to change in input when evaluated on some baseline.Given access to different attribution methods, which one is the optimal choice for what purpose remains an unexplored area. Visual comparisons, though intuitive and straightforward, remains less objective since 1) humans' ideas  themselves do not accord at the most of time. 2) attribution maps generated by different methods may vary or even cause contradictory interpretations (see Fig \ref{fig:lady_or_dog} for example).
% and in some instances define quantitative evaluation metrics
% indicating preference of one attribution method over another. Ablation-based criteria such as
% \emph{Area Over Perturbation Curve}\cite{7552539} and similar\cite{alex2016layerwise,
%   Montavon_2018} tested by interventions: an attribution should
% point out inputs that, when dropped or ablated while keeping all other inputs fixed, induce the
% greatest change in output. Alternatively, measures such as \emph{Average \%
%   Drop}~\cite{Chattopadhay_2018} instead determine to what extent important inputs stand on their
% own by comparing model scores relative to scores on just the important inputs (all other inputs are
% perturbed/ablated). 
% Finally, scaling criteria such as \textit{completeness}~\cite{sundararajan2017axiomatic}, 
% % \textit{sensitivity-n}~\cite{ancona2017better}, 
% % \textit{linear-agreement}~\cite{leino2018influencedirected, sundararajan2017axiomatic} calibrate attribution to the change in output as
% % compared to change in input when evaluated on some baseline.

While evaluation criteria endow attributions with some limited semantics, the variations in design goals, evaluation metrics, and the underlying methods resulted
in attributions failing at their primary goal: aiding in model interpretation. This work alleviates
these problems and makes the following contributions.
\begin{itemize}
  \item{} We decompose and organize existing attribution methods' goals along two complementary
    properties: ordering and proportionality. While ordering requires that an attribution should
    order input features according to some notion of importance, proportionality stipulates also a
    quantitative relationship between a model's outputs and the corresponding attributions in that particular ordering.
  \item{} We describe how all existing methods are motivated by an attribution ordering
    corresponding roughly to the logical notion of necessity which leads to a corresponding
    sufficiency ordering not yet fully discussed in literature.
  \item{} We show that while some attribution methods show great performance in necessity while others show more about sufficiency but no evaluated method in this paper can be a winner on the necessity and sufficiency at the same time.
 
  \item{} We further demonstrate how to interpret different attribution maps to gain more insights about the decision making process in deep models. 
    
%   \item{} We define quantitative measurements of an attribution method's ordering and ordered
%     proportionality. \textbf{tentatively:} We also show that while perfect attribution for the
%     stronger related notion of sensitivity is not realizable for anything but linear models, our
%     specialized notion of proportionality (sensitivity) under an ordering is realizable.
%   \item{} We evaluate attribution methods according to necessity and sufficiency ordering and
%     proportionality, comparing them to each other and to optimal attributions under each ordering.
\end{itemize}

\section{Background}\label{sec:background}

 Attributions are a simple form of model explanations that have found
significant application to Convolutional Neural Networks (CNNs) with their ease of visualization alongside model inputs
(i.e. images). We summarize the various approaches in Section~\ref{sec:attribution} and the
criteria with which they are evaluated and/or motivated in Section~\ref{sec:criteria}.
\begin{figure*}[th]
  \begin{center}
    % \fbox{\rule{0pt}{2in} \rule{.9\linewidth}{0pt}}
    \includegraphics[width=0.78\linewidth]{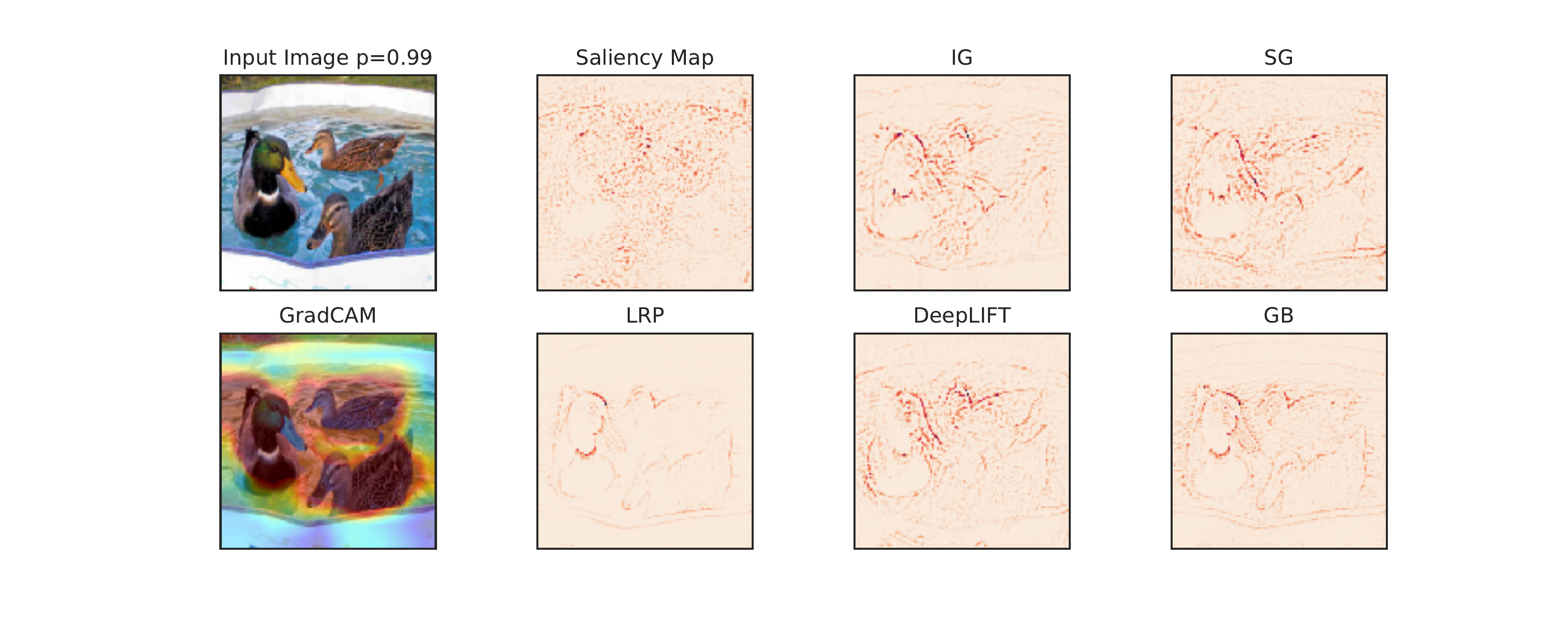}
  \end{center}
  \caption{Visualizations of different attribution methods using VGG16 model
    \cite{simonyan2014deep}. (a) is the input image with an output confidence 0.99 for the true
    label "duck". (b)-(i): different attribution methods discussed in Section \ref{sec:attribution}. \textit{grad} $\times$ \textit{input} is
    applied to (b), (c), (g), (h) and (i) to find attribution scores, while (d), (e) and (f)
    consider their results as attribution scores already. Use heatmap for (d) as the choice in \cite{selvaraju2016gradcam}.}
  \label{fig:example_duck}
\end{figure*} 
\subsection{Attribution Methods}\label{sec:attribution}

The concept of Attribution is well-defined in \cite{sundararajan2017axiomatic} but it excludes any method without an baseline (reference) input. We consider a relaxed version. Consider a classification model $ \mathbf{y} =
M(\mathbf{x}) $ that takes an input vector $ \mathbf{x} $ and outputs a score vector $ \mathbf{y} = [y_0,
  \cdots, y_i, \cdots, y_{n-1}]^{\top} $, where $ y_i $ is the score of predicting $ \mathbf{x} $ as class $
i $ and there are $ n $ classes in total. Given a pre-selected class $ c $, an attribution method
attempts to explain $ y_c $ by computing a score for each feature $x_i$ as its contribution toward $ y_c $. Even though each feature in $ \mathbf{x} $ may receive different attribution scores given different choice of attribution methods, features with positive attribution scores are universally explained as important part in $ \mathbf{x} $, while the negative scores indicate the presence of these features decline the confidence for predicting $ y_c $.

Previous work has made great progress in developing gradient-based attribution methods to highlight important features in the input image for explaining model's prediction. The primary question to answer is whether should we consider \textit{grad} or \textit{grad} $\times$ \textit{input} as attributions \cite{smilkov2017smoothgrad, sundararajan2017axiomatic, shrikumar2017learning, leino2018influence}. As \citet{ ancona2017better} argues \textit{grad} is \textit{local attribution} that only accounts for how tiny change around the input will influence the output of the network but \textit{grad} $\times$ \textit{input} is the \textit{global attribution} that accounts for the marginal effect of a feature towards output. We use \textit{grad} $\times$ \textit{input} as the attribution to be discussed in this paper. We briefly introduce
methods to be evaluated in this paper and examples are provided in Fig \ref{fig:example_duck}.

\newcommand{\methodformat}[1]{\textbf{#1}\xspace}
\newcommand{\methodsm}{Saliency Map\xspace}
\newcommand{\shortsm}{SM\xspace}
\newcommand{\methodgb}{Guided Backpropagation\xspace}
\newcommand{\shortgb}{GB\xspace}
\newcommand{\methodgc}{GradCAM\xspace}
\newcommand{\shortgc}{GCAM\xspace}
\newcommand{\methodlrp}{Layer-wise Relevance Propagation\xspace}
\newcommand{\shortlrp}{LRP\xspace}
\newcommand{\methoddl}{DeepLift\xspace}
\newcommand{\shortdl}{DL\xspace}
\newcommand{\methodig}{Integrated Gradient\xspace}
\newcommand{\shortig}{IG\xspace}
\newcommand{\methodsg}{SmoothGrad\xspace}
\newcommand{\shortsg}{SG\xspace}
\newcommand{\methodid}{Influence-Directed Explanation\xspace}
\newcommand{\shortid}{ID\xspace}
\newcommand{\methodcam}{Class Activation Map\xspace}
\newcommand{\shortcam}{CAM\xspace}

\methodformat{\methodsm} (SM) \cite{simonyan2013deep, baehrens2009explain} uses the gradient of the class of interests with
respect to the input to interpret the prediction result of CNNs. \methodformat{\methodgb} (\shortgb) \cite{springenberg2014striving} modifies the
backpropagation of ReLU~\cite{pmlr-v15-glorot11a} so that only the positive gradients will be
passed into the previous layers. \methodformat{\methodgc} \cite{selvaraju2016gradcam} builds on the \methodcam (\shortcam)
\cite{zhou2015cnnlocalization} targeting CNNs. Although its variations
\cite{Chattopadhay_2018,omeiza2019smooth} show sharper visualizations, their fundamental concepts
remain unchanged. We consider only \methodgc in this paper. \methodformat{\methodlrp} (\shortlrp) \cite{Bach2015OnPE}, \methodformat{\methoddl} \cite{shrikumar2017learning} modifies the local gradient and rules of backpropagations. Another method sharing similar motivation in design with \methoddl is \methodformat{\methodig} (\shortig)
\cite{sundararajan2017axiomatic}. \shortig computes attribution by integrating the gradient over a
path from a pre-defined baseline to the input. \methodformat{\methodsg} (\shortsg) \cite{smilkov2017smoothgrad} attempts to denoise the result of \methodsm by adding Gaussian noise to the input and provides visually sharper results.

Other methods like Deep Taylor Decompostion~\cite{montavon2015explaining} related with LRP, Occluding~\cite{zeiler2013visualizing} and Influence Directed Explanations \cite{leino2018influencedirected} are not evaluated in this paper but will be a proper future work to discuss.

\subsection{Assumptions}
We restrict ourselves with two assumptions with regards to models and attribution methods analyzed.%we make concerning the attribution methods we examine.

\noindent \textbf{Non-linearity}
We focus on evaluating the performance of attribution methods on non-linear model, \eg neural networks, as SM, IG, SG, LRP, and DeepLIFT are equivalent for linear models (see proofs in Appendix \textcolor{red}{I}) while GradCAM only works for convolutional layers. Linear models are therefore not expected to distinguish most attribution methods.

\noindent \textbf{Feature Interaction}
Features may or may not influence the decision individually. In this paper, we focus on attribution methods that are not directly suited to reasoning about feature interaction: their attribution maps represent per-pixel importance, and do not indicate relationships between pixels. We are interested in evaluating the feature interactions in the future work.

\pxm{TODO: update this}

\subsection{Evaluation Criteria}\label{section:2}\label{sec:criteria}

Evaluation criteria measure the performance of attribution methods towards some desirable
characteristic and are typically employed to justify the use of novel attribution methods. We begin with discussing two assumptions about evaluating the attribution methods. 

The most common evaluations are based on pixel-level interventions or perturbations. These quantify the
correlation between the perturbed pixels' attribution scores and the output
change~\cite{unifyarticle, alex2016layerwise, Chattopadhay_2018, gilpin2018explaining,
  Montavon_2018, 7552539, shrikumar2017learning}. For perturbations that intend to remove or ablate
a pixel (typically by setting it to some baseline or to noise), the desired behavior for an optimal
attribution method is to have perturbations on the highly attributed pixels drop the class score
more significantly than on the pixels with lower attribution.

Quantification of the behavior described by \citet{7552539} with \emph{Area Over Perturbation Curve
  (AOPC)} measures the area between two curves: the model's output score against the number
of perturbed pixels in the input image and the horizontal line of the score at the model's original
output. Two similar measurement are \emph{Area Under Curve (AUC)} \cite{alex2016layerwise, Montavon_2018} and \emph{MOst Relevant features First (MoRF)} \cite{ancona2017better} that measure the area under the perturbation curve instead. AOPC and AUC (we use AUC to represent both AUC and MoRF) measurement are equivalent and both are orignally used to endorse \shortlrp. For reasons which will become clear in the next section, we
categorize these criteria as supporting necessity order. We argue that evaluating attribution methods only with perturbation curves, \eg \emph{Area Under Curve (or AUC)}, only discovers the tip of the iceberg and potentially can be problematic. A toy model is shown in Example \ref{example:max} to elaborate our concerns.

% \textbf{Average \% Drop} measures how much the score increases but ignore how fast it increases
% because it is used for measure if the explanation localizes important areas.

%Both methods are forms of ablation in that the perturbations performed start from an input
%instance and ablate (set to baseline value) pixels in order of importance. Viewing ablation as
%removing of features, the measures are examples of a neccessity order.

%more like a necessity measurement because creating perturbations in the input image only reveals
%the necessary features but not sufficient features. We will discuss more about necessity and
%sufficiency in Section~\ref{sec:methods}. Consider the following example as a dilemma for AOPC.
%\pxm{We haven't yet discussed necessity, relativity, sufficiency in sufficient detail.}

\begin{example}\label{example:max}
  Consider a model $M(\mathbf{x}) = max(x_1, x_2) $ that takes a vector $\mathbf{x}$ with three
  features $x_1, x_2, x_3 \in \{0, 1\}$. Given the input to the model is $x_1=x_2=x_3=1$, assume
  $A_1, A_2, A_3$ are three different methods and output the attribution scores $s_1, s_2, s_3$
  shown in Table \ref{tab:example1} for each input feature $x_1, x_2, x_3$, respectively.
  \begin{table}[h]
    \centering
    \begin{tabular}{cccc}
      & $s_1$ & $s_2$ & $s_3$ \\ \hline
$A_1$ & $1/6$ & $1/3$ & $1/2$ \\
$A_2$ & $2/3$ & $0$   & $1/3$ \\
      $A_3$ & $2/3$ & $1/3$ & $0$   \\
    \end{tabular}
    \caption{$s_1, s_2, s_3$ are attribution scores for $x_1, x_2, x_3$ computed by $A_1, A_2, A_3$, respectively.}
    \label{tab:example1}
  \end{table}{}
  \begin{figure}[h!]
    \centering
    \includegraphics[width=0.8\linewidth]{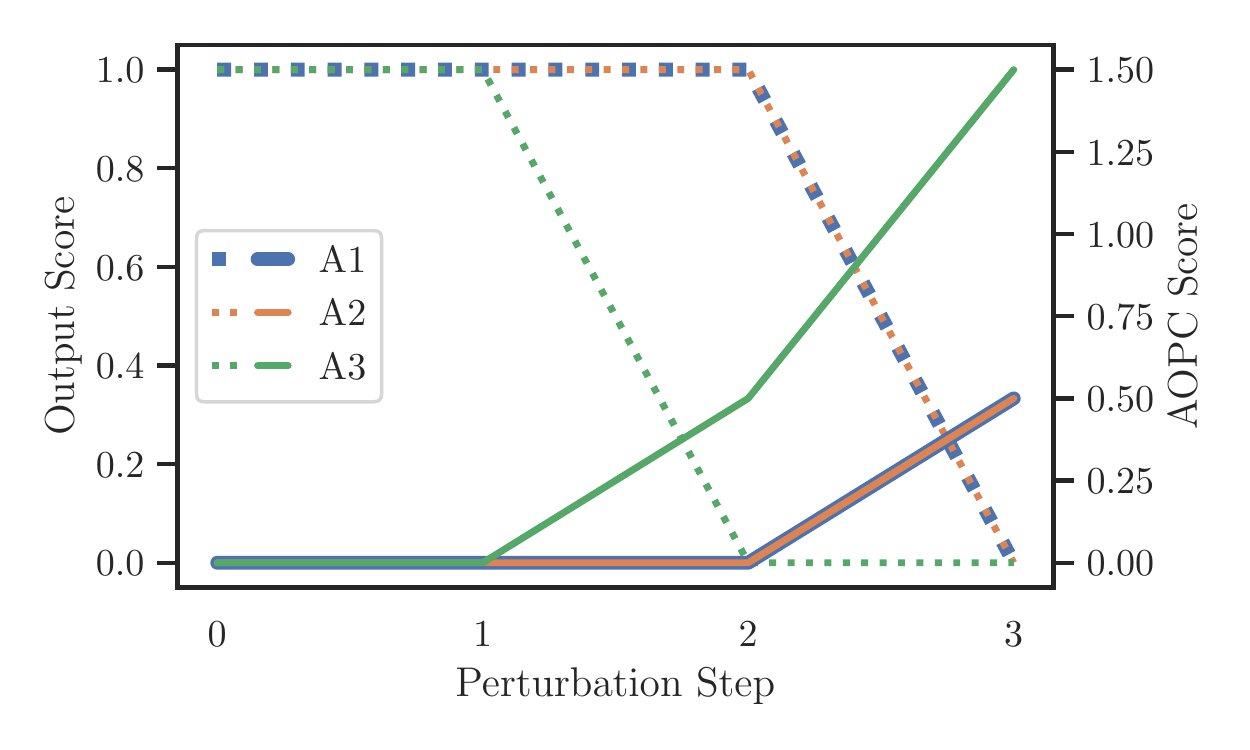}
    \caption{Comparing attribution methods $A_1, A_2$ and $A_3$ by applying zero perturbation. Dash
      lines are the change of model's output at each perturbation step (only one feature is set to 0 at each step). Solid lines
      are the changes of AOPC scores. $A_2$ and $A_3$ are overlapping with each other in this
      exmaple.}
    \label{fig:example2_figure}
  \end{figure}
  
  We apply zero perturbation to the input which means we set features to 0. The AOPC evaluation for
  these three attribution methods is shown in Fig \ref{fig:example2_figure}. Using the conclusion
  from \cite{7552539} that higher AOPC scores suggest higher relativity of input features
  highlighted by an attribution method, Fig \ref{fig:example2_figure} shows pixels highlighted by $A_3$ are more relative \citet{7552539} to prediction than $A_1$ and $A_2$, as expected. However, $A_2$ and $A_1$ are considered as showing
  same level of relativity under the AOPC measurement even though $A_2$ succeeds in discovering $a$ is more
  relevant than $c$, whereas $A_1$ believes $c$ is more relevant than both $a, b$.
\end{example}

Another set of criteria instead stipulate that positively attributed features should stand on their
own independently of non-important features. An example of this criterion is \emph{Average \% Drop} \cite{Chattopadhay_2018} in support of GradCam++ that measures the change of class score by
presenting only pixels highlighted by an attribution only (non-important pixels are ablated). Another example is the \emph{LEast Relevant Features first (LeRF)} \cite{ancona2017better} that removes the features with least high attribution scores. first. We will say these two criteria
support sufficiency order (definition to follow). 

Rethinking the concept of relativity, we
believe both necessity and sufficiency can be treated as different types of relativity. In
Example~\ref{example:max}, neither $ x_1 $ nor $ x_2 $ is a necessary feature individually because the output will not change if any one of them is absent. However, both $ x_1 $ and $ x_2 $ are sufficient features, with either of which, the model could produce the same output as before. Besides, $ A_2 $ succeeds in placing the order of sufficient feature $ x_1 $ in front of the non-sufficient feature $ x_3 $ but $ A_2 $ fails, while AOPC(or AUC) is unable to discover the success. 

Other evaluation criteria exist, like \textit{sensitivity-n}~\cite{ancona2017better} and sanity check \cite{adebayo2018sanity}, will be discussed in Section \ref{sec:related work}.
\section{Methods}\label{sec:methods}

To tame the zoo of criteria, we organize and decompose them into two
aspects: (1) \textit{ordering} imposes conditions under which an input should be more important than another
input in a particular prediction, and (2) \textit{proportionality} further specifies how attribution should
be distributed among the inputs. We elaborate on ordering criteria in Section~\ref{sec:order} with
instantiating in Section~\ref{sec:order-necessity} and Section~\ref{sec:order-sufficiency}. We
describe proportionality in Section~\ref{sec:proportionality}. We begin with the logical notions of
necessity and sufficiency as idealized versions of ablation-based measures described in
Section~\ref{sec:background}. We introduce our notations in this paper before any further discussion,
\subsection{Notation}  
\label{sec:notation} 
Consider a model $y = f(\mathbf{x})$ and an attribution method $A$, it computes a set of
attribution scores $s_1, s_2, ..., s_n$ for each pixel $x_1, x_2, ..., x_n$ in the input image $\mathbf{x}$ attributing a given class \footnote{we omit the notation of the class of interest for the simplicity in the rest of the paper}.
We permute the pixels into a new ordering $\pi_A(\mathbf{x}) = [x'_1, x'_2, ..., x'_n]$ so that $s'_1 \geq
s'_2 \geq ... \geq s'_n $. We take the subset $\pi^+_A(\mathbf{x})$ of $\pi_A(\mathbf{x})$ so that $\pi^+_A(\mathbf{x})$ has
the same ordering as $\pi_A(\mathbf{x})$ but only contains pixels with positive attribution scores. Let
$R_i(\mathbf{x}, \pi)$ be the output of the model with input $\mathbf{x}$ where pixels $x'_1, x'_2, ..., x'_i \in
\pi$ are perturbed from the input by setting $x'_1 = x'_2 = ... = x'_i = b$, where $ b $ is a
baseline value for the image (typically $b=0$). Also, let $\mathbf{x}_b$ be the the baseline input image where all the pixels are filled with the baseline value $b$. Therefore, we have the the original output $y_0 = f(\mathbf{x})$ and the baseline output $y_b=f(\mathbf{x}_b)$.

\subsection{Logical Order}\label{sec:order}
The notions of necessity and sufficiency are commonly used characterizations of logical conditions. A
necessary condition is one without which some statement does not hold.
% ; that is, where the condition
% to be made untrue while preserving the truth of other conditions, a statement at hand would be
% invalidated. 
For example, in the statement $ P_1 = A \wedge B $, both $ A $ and $ B $ are necessary
conditions as each independently would invalidate the statement were they be made false. On the
other hand, a sufficient condition is one which can independently make a statement true without other conditions being true. 
% In the statement $ P_2 = A \vee B $, both $ A $ and $ B $ are
% sufficient but neither are necessary.
In more complex statements, no atomic condition may be necessary nor sufficient though compound
conditions may. In the statement $ P_3 = (A \wedge B) \vee (C \wedge D) $, none of $ A, B, C, D $
are necessary nor sufficient but $ (A \wedge B) $ and $ (C \wedge D) $ are sufficient. As we are
working in the context of input attributions, we relax and order the concept of necessity and
sufficiency for atomic conditions (individual input pixels).

\begin{definition}[Logical Necessity Ordering]\label{def:logical-necessity}
  Given a statement $ P $ over some set of atomic conditions, and two orderings $ a $
  and $ b $, both ordered sets of the conditions, we say $ a $ has better necessity ordering for
  $ P $ than $ b $ if:
  \begin{equation}
   \min_i\paren{ \set{a_k}_{k \geq i} \cancel{\models} P } \leq \min_i \paren{ \set{b_k}_{k \geq
    i} \cancel{\models} P}       
  \end{equation}

\end{definition}
\begin{definition}[Logical Sufficiency Ordering]\label{def:logical-sufficiency}
  Likewise, $ a $ has better sufficiency ordering for $ P $ than $ b $ if:
  \begin{equation}
         \min_i \paren{\set{a_k}_{k \leq i} \models P} \leq \min_i \paren {\set{b_k}_{k \leq i} \models
    P} 
  \end{equation}
\end{definition}
A better necessity ordering is one that invalidates a statement $ P $ by removing the shorter prefix of the ordered conditions while a better sufficiency ordering is the one that can validate a statement using the shorter prefix.

\subsection{Necessity Ordering (N-Ord)}\label{sec:order-necessity}

Unlike logical statements,
numeric models do not have an exact notion of a condition (feature) being present or not. Instead,
inputs at some baseline value or noise are viewed as having a feature removed from an input. Though
this is an imperfect analogy, the approach is taken by every one of the measures described in
Section~\ref{sec:background} that make use of perturbation in their motivation. Additionally, with numeric outputs, the nuances in output obtain magnitude and we
can longer describe an attribution by a single index like the minimal index of
Definitions~\ref{def:logical-necessity} and \ref{def:logical-sufficiency}. Instead we consider an
ideal ordering as one which drops the numeric output of the model the most with the least number of inputs
ablated. 

We refer the AUC measurement \cite{alex2016layerwise,
  Montavon_2018, } and MoRF \cite{ancona2017better} as means to measure the Necessity Ordering (N-Ord). Denote $N_o(\mathbf{x}, A)$ as N-Ord score given a input image $\mathbf{x}$ and an attribution method $ A $. Rewrite AUC using the notation in Section \ref{sec:notation}:
  \begin{equation}
    N_o(\mathbf{x}, A) =\frac{1}{M+1}\sum^M_{m=0}R^m_0(\mathbf{x}, A)\label{eq:1}
  \end{equation}
where $R^m_0(\mathbf{x}, A) = max\{R_m(\mathbf{x}, \pi^+_A(\mathbf{x}))-y_b, 0\}$ and $M$ is the total number of pixels in $\pi^+_A(X)$. We include $max$ to clip scores below the baseline output. According to Definition \ref{def:logical-necessity}, we have the following proposition. 
\begin{proposition}
  An attribution method $A_1$ shows a (strictly) better Ordering Necessity than another method
  $A_2$ given an input image $\mathbf{x}$ if $N_o(\mathbf{x}, A_1) < N_o(\mathbf{x}, A_2)$ \label{property: NO}
\end{proposition}

As discussed in Section
\ref{section:2}, N-Ord only captures whether more necessary pixels, are receiving higher attribution scores. We argue that attribution methods should also be
differentiated by the ability of highlighting sufficient features. To evaluate whether more sufficient pixels are receiving higher attribution scores, we propose
Sufficiency Ordering as a complementary measurement. 

\subsection{Sufficiency Ordering (S-Ord)}\label{sec:order-sufficiency}
We believe LeRF \cite{ancona2017better} is a related means of measuring the Sufficiency. Sufficiency Ordering measures the score increase as we keep adding important features into a baseline input. Use the notation in Section \ref{sec:notation} and let $ R'_i(\mathbf{x}_b, \pi) $ be the model's output with $ \mathbf{x}_b $
where $x'_1, x'_2, \cdots , x'_i \in \pi $ are added to the baseline image $\mathbf{x}_b$. Denote $S_o(\mathbf{x}, A)$ as S-Ord score given a input image $\mathbf{x}$ and an attribution method $ A $.
\begin{equation}
    S_o(\mathbf{x}, A) =\frac{1}{M+1}\sum^M_{m=0}R^{m'}_0(\mathbf{x}, A) \label{eq:So}
\end{equation}
where $R^{m'}_0(\mathbf{x}, A) = min\{R'_m(\mathbf{x}_b, A), R'_M(\mathbf{x}_b, \pi^+_A(\mathbf{x}))\}-y_0$, $M$ is the number of pixels in $\pi^+_A(\mathbf{x})$. We include $min$ to clip scores above the original output. According to Definition \ref{def:logical-sufficiency}, we have the following proposition. 
\begin{proposition}
\label{property 2}
  An attribution method $A_1$ shows (strictly) better Ordering Sufficiency than another method
  $A_2$ given an input image $X$ if $S_o(\mathbf{x}, A_1) > S_o(\mathbf{x}, A_2)$.
\end{proposition}

N-Ord and S-Ord together provides a more
comprehensive evaluation for an attribution method. In Section \ref{Sec: Proportionality}, we are
going to discuss the disadvantages of only using N-Ord or S-Ord and propose Proportionality as a refinement to the ordering analysis.

\subsection{Proportionality} \label{Sec: Proportionality}\label{sec:proportionality}
N-Ord and S-Ord do not incorporate the attribution scores beyond producing an ordering. This can be an issue toward an accurate description of feature necessity or sufficiency. For example, consider a toy model $M(x_1, x_2) = 2x_1 + x_2$ and let the inputs variables be $x_1=x_2=1$. Any attribution methods that assign higher score for $x_1$ than $x_2$ produces the identical ordering $\pi(x_1, x_2) = [x_1, x_2]$, even one could overestimate the degree of necessity (or sufficiency) of $x_1$ by assigning it with much higher attribution scores. With \textit{linear agreement} \cite{leino2018influence}, scores for $x_1$ and $x_2$ are more reasonable if their ratio is close to 2:1. Explaining a decision made by a more complex model only using ordering of attributions may overestimate or underestimate the necessity (or sufficiency) of an input feature. Therefore, We propose Proportionality as a refinement to quantify the necessity and sufficiency in complementary to the ordering measurement.

\theoremstyle{definition}
\begin{definition}[Proportionality-k for Necessity]
\label{def: proportionality-k for necessity}

Consider two positive number $n_1, n_2$ and an attribution method A. Use notations in Section \ref{sec:notation} and let $\hat{\pi}^+_A(\mathbf{x})$ be a reversed ordering of $\pi^+_A(\mathbf{x})$. Proportionality-k for Necessity is measured by
\begin{equation}
    N^k_p(\mathbf{x}, A) = |R_{n_1}(\mathbf{x}, \pi^+_A(\mathbf{x})) - R_{n_2}(\mathbf{x}, \hat{\pi}^+_A(\mathbf{x}))| \label{eq:2}
\end{equation}
under the condition $\sum^{n_1}_i s_i = \sum^{n_2}_j s_j = k S(A, \mathbf{x}), s_i \in \pi^+_A(\mathbf{x}), s_j \in \hat{\pi}^+_A(\mathbf{x}), k\in[0, 1]$. $R_i(\mathbf{x}, \pi)$ uses the same definition in \eqref{eq:1}, and $S( \mathbf{x}, A)$ is the sum of total
positive attribution scores.
\end{definition}

% Proportionality-k for Necessity measures the difference between the output of perturbing pixels by highest attribution scores first and the other output when perturbing pixels by lowest first. Two orderings contain the same share of total attribution scores but can have different number of pixels. (\eg for attribution method $A_3$ in Example \ref{Example: 2}, we compare the output of perturbing $x_2$ in $X$ with the output of perturbing $x_1$ and $x_3$ together since $s_2 / (s_1+s_2+s_3) = (s_1+s_3) / (s_1+s_2+s_3)$.

% Less pixels are perturbed if they are chosen from a descending order $\pi^+_A(X)$ to make the sum of attribution scores equal to $k S(A, X)$, while more pixels are needed if picking them from an ascending order $\hat{\pi}^+_A(X)$, starting from the lowest.

\textbf{Explanation of Definition \ref{def: proportionality-k for necessity}} the motivation behind Proportionality-k for Necessity is that: given a group of pixels ordered with their attribution scores, there are different ways of distributing scores to each feature while the ordering remains unchanged. An optimal assignment is preferred that features receive attribution scores proportional to the output change if they are modified accordingly. In other words, given any two subsets of pixels $\pi_1$ and $\pi_2$. with total attribution scores sum to $S_1$ and $S_2$, are perturbed, the change of output scores $R(\mathbf{x}, \pi_1)$ and $R(\mathbf{x}, \pi_2)$ should satisfy $R(\mathbf{x}, \pi_1)/R(\mathbf{x}, \pi_2) = S_1/S_2$. This property is demanded because the same share of attribution scores should account for the same necessity or sufficiency. If we restrict the condition to $S_1 = S_2$, the difference between $R(\mathbf{x}, \pi_1)$ and $R(\mathbf{x}, \pi_2)$ becomes an indirect measurement of the proportionality. For the measurement of Necessity, we further restrict that $\pi_1$ is perturbed from the pixel with the highest attribution score first and $\pi_2$ is perturbed from the one with lowest attribution score first, in accordance with the setup in N-Ord. Therefore, a smaller difference $N^k_p(\mathbf{x}, A)$ shows better Proportionality-k for Necessity

\begin{proposition}
  An attribution method $A_1$ shows better Proportionality-k for Necessity than method $A_2$ if $N^k_p(\mathbf{x}, A_1) < N^k_p(\mathbf{x}, A_2)$  \label{property: proportionality-k for necessity}
\end{proposition}

A similar requirement for attribution method is \textit{completeness} discussed by
\cite{sundararajan2017axiomatic} and its generalization \textit{sensitivity-n} discussed by
\cite{ancona2017better}.  \textit{completeness} requires the sum of total attribution scores to be
equal to the change of output compared to a baseline input, and \textit{sensitivity-n} requires any
subset of $n$ pixels whose summation of attribution scores should be equal to the change of output
compared to the baseline if pixels in that subset are removed. When $n$ is the total number of
pixels in the input image, \textit{sensitivity-n} reduces to \textit{completeness}. The
relationships between \textit{sensitivity-n} and Proportionality-k for Necessity are discussed as
follows:

\begin{proposition}
If an attribution method $A$ satisfies both sensitivity-$n_1$ and sensitivity-$n_2$,
then $N^k_p(\mathbf{x}, A) = 0$ under the condition if $\sum^{n_1}_i s_i = \sum^{n_2}_j s_j = k S(\mathbf{x}, A), s_i \in \pi^+_A(\mathbf{x}), s_j \in \hat{\pi}^+_A(\mathbf{x}), k\in[0, 1]$ \label{pro:1}, but not vice versa.
\end{proposition}

The proof for Proposition \ref{pro:1} and can be found in Appendix 1. We further contrast our method with \textit{sensitivity-n} in Section \ref{sec:related work}. Integrating \textit{proportionality} with all possible shares of attribution scores, we define the Total Proportionality for Necessity (TPN):

\begin{definition}[Total Proportionality for Necessity]
\label{def: TPN}

Given an attribution method A and an input image $\mathbf{x}$, The Total Proportionality for Necessity is measured by
% \begin{equation}
%     N_p(\mathbf{x}, A) =(1 + e^{\alpha r})\int^1_0 N^k_p(\mathbf{x}, A) dk \label{eq:4}
% \end{equation}
\begin{equation}
     N_p(\mathbf{x}, A) =\frac{1}{ry_0}\int^1_0 N^k_p(\mathbf{x}, A) dk \label{eq:4}
 \end{equation}
% where $r =\max(R_M(\mathbf{x}, \pi^+_A(\mathbf{x}))-B, 0)/R_0(\mathbf{x})$, 
where $r=\min \{y_b / R_M(\mathbf{x}, \pi^+_A(\mathbf{x})), 1\}$\footnote{We clip the scores below 0 and add a small positive number $\epsilon$ to the denominator to ensure the numerical stability.}. $y_0$ is used as a normalizer and  $M$ is the total number of elements in $\pi^+_A(\mathbf{x})$, therefore, $r= 1$ if removing all elements in $\pi^+_A(\mathbf{x})$ drops the score to the baseline. Revisit the Section \ref{sec:notation} for notations if needed. 
\end{definition}
% We prove that, for a linear model, if an attribution method assigns each feature $x_i$ in $X$ with a score $w_i \times x_i$ ($w_i$ is the weight for $x_i$), then $N^k_p(X, A)=0$ for all possible $k \in [0, 1]$
% \textbf{Need Appendix}. However, for a model contains non-linearity, like sigmod activations,

\noindent \textbf{Explanation for Definition \ref{def: TPN}}
$N_p(\mathbf{x}, A)$ is the area between two perturbation curves one starting from the pixels with highest attribution scores and the other with a reversed ordering. The difference from Necessity Ordering is that $N_p(\mathbf{x}, A)$ is measured against the share of attribution scores (the value of $k$) instead of the share of pixels in the $N_o(\mathbf{x}, A)$. On the other side, perturbations on non-necessary features may not change the output at all and we penalize an attribution method that guides us to do so with the ratio $r$ compared to the baseline. Generalizing Proposition \ref{property: proportionality-k for necessity}, we argue:
 \begin{proposition}
  An attribution method $A_1$ shows better Total Proportionality for Necessity than method $A_2$ if $N_p(\mathbf{x}, A_1) < N_p(\mathbf{x}, A_2)$   \label{property: total proportionality for necessity} 
 \end{proposition}

Under the similar construction, we have the following definition of Proportionality-k for Sufficiency and Total Proportionality for Sufficiency (TPS):

\begin{definition}[Proportionality-k for Sufficiency]
Consider two positive number $n_1, n_2$ and an attribution method A. Use notations in Section \ref{sec:notation} and let $\hat{\pi}^+_A(\mathbf{x})$ be a reversed ordering of $\pi^+_A(\mathbf{x})$. Proportionality-k for Sufficiency is measured by
\begin{equation}
    S^k_p(\mathbf{x}, A) = |R'_{n_1}(\mathbf{x}_b, \pi^+_A\mathbf{x})) - R'_{n_2}(\mathbf{x}_b, \hat{\pi}^+_A(\mathbf{x}))| \label{eq:7}
\end{equation}
under the condition $\sum^{n_1}_i s_i = \sum^{n_2}_j s_j = k S(\mathbf{x}, A), s_i \in \pi^+_A(\mathbf{x}), s_j \in \hat{\pi}^+_A(\mathbf{x}), k\in[0, 1]$. $R'_i(\mathbf{x}, \pi)$ reuses the definition in \eqref{eq:So}; $S( \mathbf{x}, A)$ is the sum of total positive attribution scores.
\end{definition}

We want the difference $S^k_p(\mathbf{x}, A)$ as small as possible since the same share of attribution scores should reflect same sufficiency. Therefore, we have the following proposition:
 \begin{proposition}
  An attribution method $A_1$ shows better Proportionality-k for Sufficiency than method $A_2$ if $S^k_p(\mathbf{x}, A_1) < S^k_p(\mathbf{x}, A_2)$  \label{property: proportionality-k for sufficiency}
 
 \end{proposition}

% To free the choice of $k$, we also define the Total Proportionality for Sufficiency and its property.

\begin{definition}[Total Proportionality for Sufficiency]
Given an attribution method A and an input image $\mathbf{x}$, The Total Proportionality for Sufficiency is measured by
% \begin{equation}
%     S_p(\mathbf{x}, A) = (1+e^{\beta(1-r')}) \int^1_0 S^k_p(\mathbf{x}, A) dk \label{eq:10}
% \end{equation}
\begin{equation}
    S_p(\mathbf{x}, A) = \frac{1}{r'y_0} \int^1_0 S^k_p(\mathbf{x}, A) dk \label{eq:10}
\end{equation}where
where $r' =\min\{R'_M(\mathbf{x}, \pi^+_A(\mathbf{x})) / y_0, 1\}$. $y_0$ is used as a normalizer and  $M$ is the total number of elements in $\pi^+_A(\mathbf{x})$, therefore, $r'= 1$ if adding all elements in $\pi^+_A(\mathbf{x})$ increases the score to the original output. Refer to Section \ref{sec:notation} and \ref{sec:order-sufficiency} for details about the notation.
\end{definition}

Similarly, $S_p(\mathbf{x}, A)$ is the area between curves of model's output change by adding pixels to a baseline input with the highest attribution scores first or by the lowest first. The ratio $r'$ penalizes the false postive situation when adding all pixels with positive scores does not increase the output significantly.
Finally, we have

 \begin{proposition}
   An attribution method $A_1$ shows better Total Proportionality for Sufficiency than another method $A_2$ if $S_p(\mathbf{x}, A_1) < S_p(\mathbf{x}, A_2)$  \label{property: total proportionality for sufficiency}

 \end{proposition}

In summary, we differentiate and describe the Necessity Ordering (N-Ord) and Sufficiency Ordering (S-Ord) from previous work and propose Total Proportionality for Necessity (TPN) and Total Proportionality for Sufficiency (TPS) as refined evaluation criteria for necessity and sufficiency. We then apply our measurement to explain the prediction results from an image classification task in the rest of the paper.

\section{Evaluation}\label{sec:evaluation}
% We evaluate our metrics directly on CNNs. A linear model may be a reasonable choice to begin with but as \citet{ancona2017better} concludes that SM, IG, LRP,DeepLIFT are equivalent for linear models. Their proof also applies to SG, BP and ID. GradCAM, on the other hand, is not defined for models without convolutional layers. Linear models are therefore not expected to distinguish most attribution methods
\subsection{Implementation of Proportionality}
To compute TPN for each single input, we ablate a subset of input features at each time. Different from Ordering test, we do not ablate a certain number of features, instead, we ablate a subset of features with a certain share of attribution scores. The share of attribution scores $k$ goes from 0 to 1. We generate the ablation curve from the features with highest scores first and the features with least highest scores first, and measure the area between these two curves. Optimal TPN will be 0 as discussed in the previous sections. The similar implementation will be applid to TPS. 
\begin{figure}[t]
    \begin{center}
        % \fbox{\rule{0pt}{2in} \rule{.9\linewidth}{0pt}}
        \includegraphics[width=\linewidth]{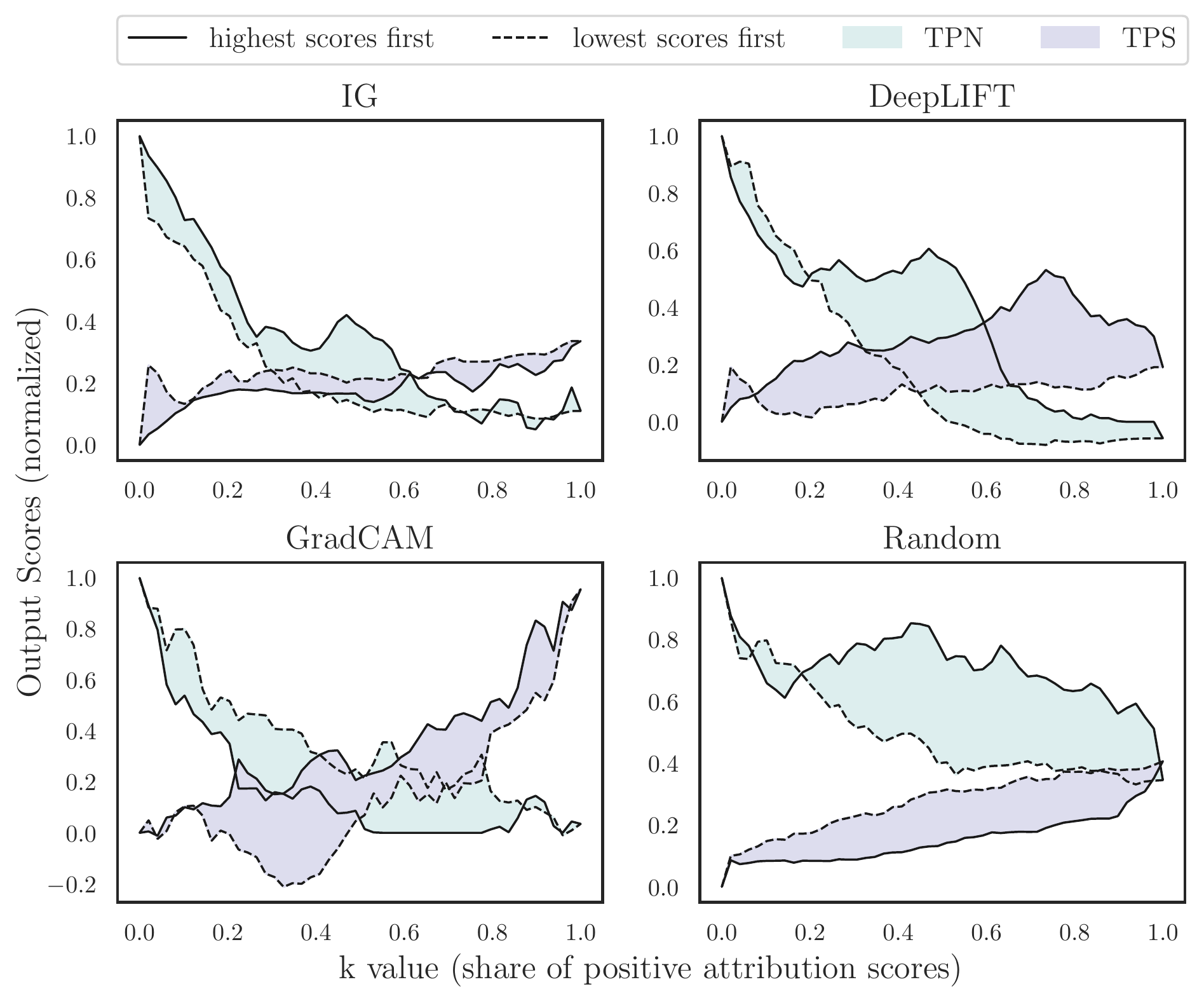}
    \end{center}
  \caption{An example of computation for TPN and TPS with the input in Fig \ref{fig:example_duck}. The area represents the score of TPN and TPS before applying the penalty $r$ and $r'$ for better visulization purpose. We also include a Random method that randomly assigns attribution scores first as a baseline method to compare with . This figure is better viewed in color. }
\label{fig:duck_NpSp}
\end{figure}

\begin{figure}
  \begin{center}
    % \fbox{\rule{0pt}{2in} \rule{.9\linewidth}{0pt}}
    \includegraphics[width=\linewidth]{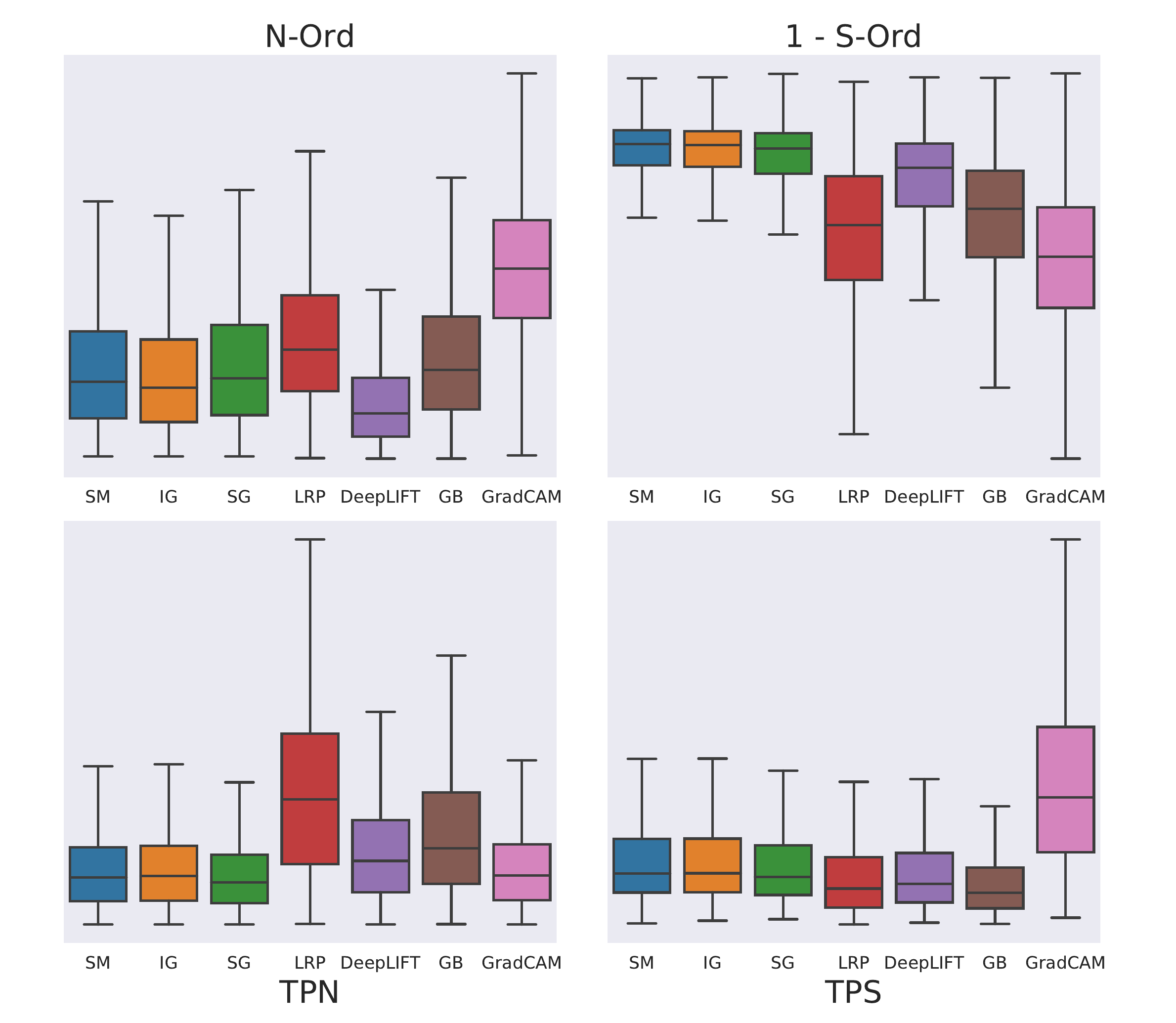}
  \end{center}
  \caption{The boxplot for different attribution methods evaluated with four crietria aforementioned on 9600 images from ImageNet with VGG16 model. Since only the higher S-Ord score indicates better Sufficiency Ordering, we use $1-$ S-Ord to accord with other criteria \textbf{so the lower scores will indicate better performance on all criteria shown above}. }
  \label{fig:heatmap}
\end{figure}

\subsection{Evaluate on the datasets}
\zifan{TODO: review and polish this}
We evaluate N-Ord, S-Ord, TPN and TPS on 9600 images sampled from ImageNet \cite{imagenet_cvpr09} on VGG16 models. We evaluate all the attribution methods motioned in Section \ref{sec:methods} and the implementation detail for the model and attribution methods can be found in the Appendix 3. We show the boxplot in Fig \ref{fig:heatmap}. We plot $1-$ S-Ord instead the original S-Ord score so that that conclusion that lower scores represent better performance holds for all subfigures. On the ImageNet and VGG model, DeepLIFT has shown relatively better Necessity Ordering while GradCAM shows relatively better Sufficiency Ordering. Considering the proportionality, Saliency Map, Integrated Gradient and Smooth Gradient are all showing slightly better proportionality for necessity while LRP and GB are showing slightly better proportionality for sufficiency. But clearly, no evaluated method is significantly better than others on all criteria at the same time.

We analyze these attribution methods based on the result of Fig \ref{fig:heatmap}:
\textbf{Saliency Map} performs not bad in necessity for both ordering and \emph{proportionality} compared to sufficiency. Therefore, the features with highest scores assigned by Saliency Map may not be actually sufficient for the decision making process, \eg the pool is highlighted by Sailency Map in Fig \ref{fig:example_duck} but the model will probably not make a mistake when only the pool is present to it. One possible reason is that the vanishing gradient issue causes the loss of gradient signal. Threfore, \textbf{Integrated Gradient and DeepLIFT} are two means trying to solve the vanishing gradient issue in Saliency Map and both archive better Necessity Ordering compared to the Saliency Map. But they do not show significantly better proportionality in both necessity and sufficiency. The reason behind this we assume is that the \emph{Summation-to-Delta} requirement only guarantees that sum of attribution scores for all features equals to the change of output, while any other share of attribution scores does not cause equivalent change to the output, so the proportionality is not improved. Similar conclusions are also discussed by \emph{sensitivity-n} \cite{ancona2017better} \textbf{Smooth Gradient} shows lower inter-quartile range in both TPN and TPS compared to Saliency Map. Computing the expectation of the Saliency Map of a distribution of inputs does not resolve the possible vanishing gradients issue for each input in the distribution; however, at the same activation unit, \eg ReLU, an input's gradient signal is blocked by flatten negative region but its neighbor's gradient signal can get unblocked. It may help to explain the improvement Smooth Gradient shows in th experiment. \textbf{GradCAM} is probably the best choices one can have for the Sufficiency Ordering regardless of the fact it does poorly in porportionality for the sufficiency -- the attribution scores may not reflect actual sufficiency. The result is not surprising because the upsampling process in GradCAM does not relate to any axiom that guarantees to produce pixel-level proportional scores.

On the contradictory, we can not make instructive comment on the following two attribution methods:
\textbf{Guided Backpropagation} shows better sufficiency on ordering and proportionality compared to the necessity. We consider it as a good method to reveal the sufficient features, however, as \citet{adebayo2018sanity} points out GB lacks fearfulness to the model by behaving poorly in the \emph{sanity check}. Therefore, we leave the understanding of GB as a future work. On the other hand, \textbf{Layer-wise Relevance Propagation} is the one we will not make much strong conclusion as well since there are many rules in LRP and only one of them, $\alpha2\beta1$-LRP (see Appendix \textcolor{red}{II}), is tested. But specifically, for $\alpha2\beta1$-LRP, it shows good sufficiency on both criteria, which increase our confidence to interpret the result of $\alpha2\beta1$-LRP as identifying sufficient features in the input space.

\subsection{Evaluate with one instance}
All the metrics discussed before can be applied to one single input and the interpretation using all winners for each criteria can provide more insights about the model. For example, in Fig \ref{fig:dogexample}, we interpret that the body of a dog is necessary to the \texttt{English springer} class and only providing the body the model may not consider it is a dog, the sufficient feature is its head. Consulting different attributions and interpret with the winners can give more comprehensive understandings. More exmaples are included in Appendix \textcolor{red}{III}. 

\begin{figure}[h]
  \begin{center}
    % \fbox{\rule{0pt}{2in} \rule{.9\linewidth}{0pt}}
    \includegraphics[width=\linewidth]{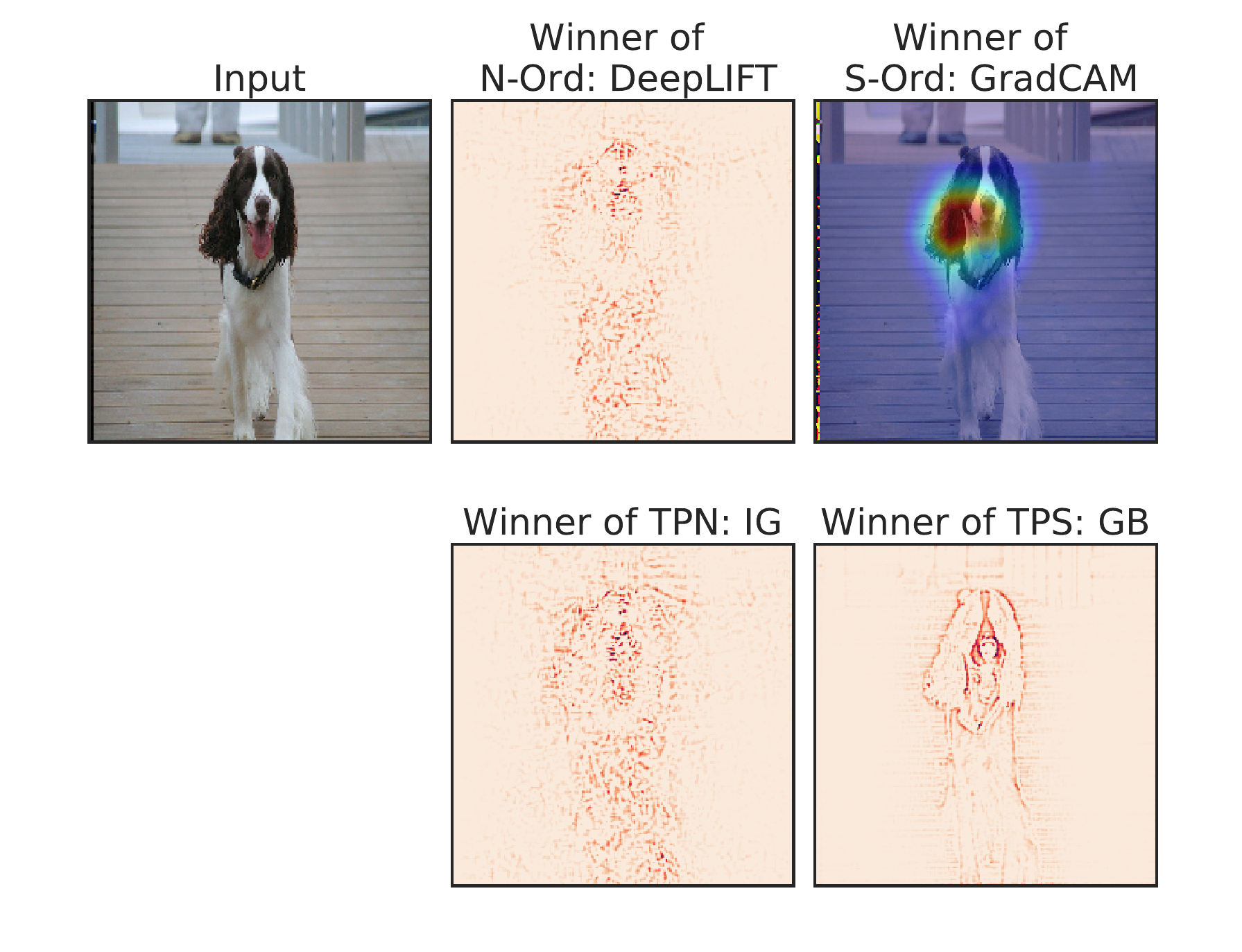}
  \end{center}
  \caption{An example of interpreting the model's predictions with winners on different criteria. }
  \label{fig:dogexample}
\end{figure}

% Doubts are raised whether gradient-based attribution methods are preferring contours when applied to CNNs \cite{adebayo2018sanity}. S-ord analysis in Fig \ref{fig:heatmap} has shown GradCAM and GB are much better at presenting sufficient features, indicating 

% In this paper, we provide two realizable purposes of interpretation: identifying more necessary or more sufficient features. The purposes can be realized by using attribution methods winning the necessity or sufficiency test, respectively. Failing to be a better attribution method in showing necessity does not mean an method can not be showing insight on the sufficiency side, \eg GB is doing poorly in N-Ord test but fairly good in S-Ord test. 

% \zifan{Revisit this} For example, GradCAM changes from one of the most frequent winners in S-Ord to the most frequent 6th in TPS. Revisiting the algorithm behind GradCAM, it computes scores based on the activation of the selected layer weighted sum by the gradients. We believe GradCAM is outstanding in localizing which part is more sufficient in the image but we find it hard to justify the activation of an internal layer is proportional to the necessity or sufficiency of an input feature.

\section{Related Work}
\label{sec:related work}
To the best of our knowledge, we are the first one describing the concepts of necessity and sufficiency in attributions where similar work may only touch the surface of either necessity or sufficiency but not both. Our work is partially motivated by \emph{smallest sufficient region} (SSR) and \emph{smallest destroying region} 
(SDR) \cite{dabkowski2017real} where the authors aim to propose a region that either increase or decrease the model's output most even though SSR and SDR only capture the spatial location in the image but do not incorporate the feature contributions as scores.

We also consider our work as a subset of \textbf{sensitivity} evaluation that how well we can trust an attribution method with its quantification of the feature importance in the input. A close concept is \textit{quantitative input influence} by \citet{7546525} (even though the author does not target on deep neural networks). \textit{sensitivity (a)(b)} \cite{sundararajan2017axiomatic} provides the basis of discussion and \textit{sensitivity-n} \cite{ancona2017better} imposes more strict requirements. The mathematical connection between \textit{proportionality} with \textit{sensitivity-n} is discussed in Section \ref{sec:proportionality}. We discuss the main difference of these two concepts here. \textit{proportionality} approaches the sensitivity from a view that, regardless of the number of pixels, same share of attribution scores should account for same 
change to the output, while \textit{sensitivity-n} requires removing $n$ pixels should change the output by the amount of total attribution scores of that $n$ pixels. \textit{sensitivity-n} only provides \textit{True/False} to an attribution methods, but \textit{proportionality} provides numerical results for comparing different methods under two purposes, the necessity and sufficiency.

% Beyond sensitivity, the \textbf{continuity} of attribution methods is an important and desired property, which requires an attribution method can to output similar results for similar input and prediction pairs. \citet{ghorbani2017interpretation} discussed this property , \citet{Montavon_2018} and \citet{kindermans2017unreliability} provide failure cases and well-designed attacks that causes unreasonable attribution results. Besides, \cite{adebayo2018sanity} evaluates the correlation between attribution scores with the model's parameters. Outside the discussion around CNNs, \cite{arras2019evaluating} provides evaluations on applying attribution methods to language tasks with LSTM \cite{Hochreiter:1997:LSM:1246443.1246450}.
%\section{Discussion}
\section{Conclusion}
\label{sec: discussion}
\zifan{Update this}
% In this section, we discuss how this paper helps to correct some potential misunderstanding from the interpretation of attribution methods. We argue meaningful interpretations should depend on the purpose of interpretation and the criteria in evaluation.

\pxm{TODO: update}

% We provide ordering and proportionality as two criteria in this paper. Interpretation based on ordering criteria is safe to argue the features with higher attribution scores in the input image is more necessary or sufficient than others, but is not safe to include any quantitative analysis, \eg  feature $x_1$ is as twice necessary as $x_2$ or $x_1$ is equivalently necessary as the presence of $x_2$ and $x_3$ together because they have similar share of attribution scores. Interpretations based on proportionality, however, provides more confidence in making quantitative argument with the winner attribution method.

In this work, we summarized existing evaluation metrics for attribution methods and categorized them into two logical concepts, necessity and sufficiency. We then demonstrated realizable criteria to quantify necessity and sufficiency with an analysis focused on ordering and its refinement, proportionality.
We evaluated existing attribution methods against our criteria and listed the best methods for each criteria. We discovered that certain attribution methods excel in necessity or sufficiency, but none is a frequent winner for both. 

The logical concepts of necessity and sufficiency are generally mutually exclusive and our analogues show the same based on our results: no method is universally optimal for both necessity and sufficiency. While this means we cannot endorse one method over others, the techniques we present provide additional interpretability tools to data scientist who can use our measures to select the attribution appropriate to the task at hand. When debugging a model for identifying traffic stop signs, an analyst can select for methods with greater necessity to determine whether the model has learned spurious correlates, e.g., the pole holding up the sign. A ``necessary'' pole would lead to false negatives (stop signs not on poles) while a ``sufficient'' one would only indicate potential false positives (poles without stop signs) which, though also problematic, are not as dangerous as false negatives in this case. The increased basis with which to interpret attribution will hopefully lead to a fuller understanding of model behaviour.

% \zifan{Update this}

\section*{Acknowledgement}
This work was developed with the support of NSF grant CNS-1704845 as
well as by DARPA and the Air Force Research Laboratory under
agreement number FA8750-15-2-0277. The U.S. Government is authorized
to reproduce and distribute reprints for Governmental purposes not
withstanding any copyright notation thereon. The views, opinions,
and/or findings expressed are those of the author(s) and should not
be interpreted as representing the official views or policies of
DARPA, the Air Force Research Laboratory, the National Science
Foundation, or the U.S. Government. 

We gratefully acknowledge the support of NVIDIA Corporation with the donation of the Titan V GPU used for this research. 

% \pxm{Deleted one GPU. nvidia doesn't want any project to get more than 1 gpu grant.}

{\small
    \bibliographystyle{ieee_fullname}
    \bibliography{egbib}
}

\section*{Appendix}
\subsection*{Appendix I: Proofs}
\label{app:i}

\noindent\textbf{Nonlinearity}
\citet{ancona2017better} concludes that SM, IG, LRP,DeepLIFT are equivalent for linear models and their proof also applies to SG. We first introduce the following proposition:

\begin{proposition}
All attribution methods mentioned in Sec \ref{sec:background} except GradCAM and Guided Backpropagation are equivalent if the model behaves linearly. 
\end{proposition}
\begin{proof}
As the Proposition 4 and Conclusion 6 in \citet{ancona2017better} prove that Saliency Map, Integrated Gradient, DeepLIFT and LRP are equivalent for a linear model, we just need to prove SmoothGrad is equivalent to Saliency Map if the model is linear. 

If a model behaves linearly, we can express the output score $y_c$ for class $c$ as a linear combination such that $y_c = \mathbf{w}^\top_c \mathbf{x} + b_c$. Then the SmoothGrad $S(\mathbf{x})_c$ is 
\begin{equation}
\begin{aligned}
 S(\mathbf{x})_c &= \mathbb{E}_{\boldsymbol{\epsilon}\sim \mathcal{N}(\mathbf{0}, \mathbf{I}\lambda)} \frac{\partial[ \mathbf{w}^\top_c (\mathbf{x} + \boldsymbol{\epsilon}) + b_c]}{\partial \mathbf{x}} \\
        &= \mathbb{E}_{\boldsymbol{\epsilon}\sim \mathcal{N}(\mathbf{0}, \mathbf{I}\lambda)}\mathbf{w}_c\\
        & = \mathbf{w}_c = \frac{\partial y_c}{\partial \mathbf{x}} \quad(\text{Saliency Map})
\end{aligned}
\end{equation}
\end{proof}

\noindent\textbf{Proof to Proposition 4}
\textit{If an attribution method $A$ satisfies both sensitivity-$n_1$ and sensitivity-$n_2$,
then $N^k_p(\mathbf{x}, A) = 0$ under the condition if $\sum^{n_1}_i s_i = \sum^{n_2}_j s_j = k S(\mathbf{x}, A), s_i \in \pi^+_A(\mathbf{x}), s_j \in \hat{\pi}^+_A(\mathbf{x}), k\in[0, 1]$ \label{pro:1}, but not vice versa}.

\begin{proof}
If A satisfies sensitivity-$n_1$, for any given ordered subset $\pi$, we have $$\sum^{n_1}s_i = R(\mathbf{x},\pi)$$ 
Same thing happens to $n_2$ if A satisfies sensitivity-$n_2$. Under the condition if $\sum^{n_1}_i s_i = \sum^{n_2}_j s_j = k S(\mathbf{x}, A), s_i \in \pi^+_A(\mathbf{x}), s_j \in \hat{\pi}^+_A(\mathbf{x}), k\in[0, 1]$, 
\begin{equation}
\begin{aligned}
    N^k_p(\mathbf{x}, A) &= |R(\mathbf{x},\pi_A(\mathbf{x}) - R(\mathbf{x},\pi^+_A(\mathbf{x})|\\
    &=|\sum^{n_1}_i s_i - \sum^{n_2}_j s_j|\\
    &=|k S(\mathbf{x}, A) - k S(\mathbf{x}, A)| = 0
\end{aligned}   
\end{equation}

\end{proof}

\newpage
\subsection*{Appendix II: Implementation Details}
\label{implementaiton detail}

\subsubsection*{Models}

We evaluate N\_Ord, S\_Ord for all attribution methods mentioned in Section \ref{sec:background}. We evaluate on 9600 images from ImageNet \cite{imagenet_cvpr09} with pre-trained on VGG16\cite{simonyan2014deep}. 

\subsubsection*{Attribution Methods}
\noindent \textbf{Saliency Map}

As discussed in Sec \ref{sec:background}, we use \texttt{grad} $\times$ \texttt{input} to represent the Saliency Map, instead of the vanilla gradient.

\noindent \textbf{Integrated Gradient}

We use the black image as the baseline for all images and we use the 50 samples on the linear path from the baseline to the input. 

\noindent \textbf{Smooth Gradient}

As discussed in Sec \ref{sec:background}, we use \texttt{smooth\_grad} $\times$ \texttt{input} to represent the Smooth Gradient. We pick a noise level of 20 $\%$ as it appears to be the best parameter in its original paper \cite{smilkov2017smoothgrad}. We randomly sample 50 points from the Gaussian distribution for the aggregation. 

\noindent \textbf{DeepLIFT}

We use the black image as the baseline for all images and we use the RevealCancel rule for DeepLIFT \footnote{We use the release code on https://github.com/kundajelab/deeplift}

\noindent \textbf{LRP}

We use the implementation of LRP-$\alpha2\beta1$ with generalization tricks mentioned by \citet{Montavon_2018} who argues this rule is better for image explanations. 

\noindent \textbf{Guided Backpropagation}

To implement Guided Backpropagation, we modify the ReLU activation in the network to filter out the negative gradient in tensorflow.
\begin{lstlisting}
@ops.RegisterGradient("GuidedBackProp")
    def _GuidedBackProp(op, grad):
        dtype = op.inputs[0].dtype
        return grad * tf.cast(grad > 0., dtype) * \
                tf.cast(op.inputs[0] > 0., dtype)
\end{lstlisting}

\noindent \textbf{GradCAM}

We use the last convolutional layer to compute the GradCAM for all images. 

% For ID, we employ \textit{instance distribution of interests} \cite{leino2018influence} (simple gradient) to visualize the top 1000 neurons in \texttt{block4\_conv3} for each class. We also use the same convolutional block in as \citet{leino2018influence} did but a deeper slice expecting for deeper feature representations (the author uses \texttt{conv4\_1}). We believe a better slice can be argued but we find slice searching remains an open question. 

% For the penalty terms in TPN and TPS, we use $\alpha=1.0$ and $ \beta=3.0$. In the experiment, we observe even random perturbation can cause serious drop in the output, so there is no need for big penalty in TPN. However for TPS, even adding all pixels with positive attribution scores, the model can produce a very low scores compared to the original ones, for some instances. Therefore, we slightly increase $\beta$ to encourage the attribution method that can locate sufficient features. 
\section*{Appendix III}
More examples of evaluating each images with TPN and TPS are shown in Fig \ref{fig:demonstration}
\begin{figure*}[t]
    \begin{center}
        % \fbox{\rule{0pt}{2in} \rule{.9\linewidth}{0pt}}
        
        \begin{tikzpicture}
        
        \node[] at (0,0) {
            \includegraphics[width=0.9\linewidth]{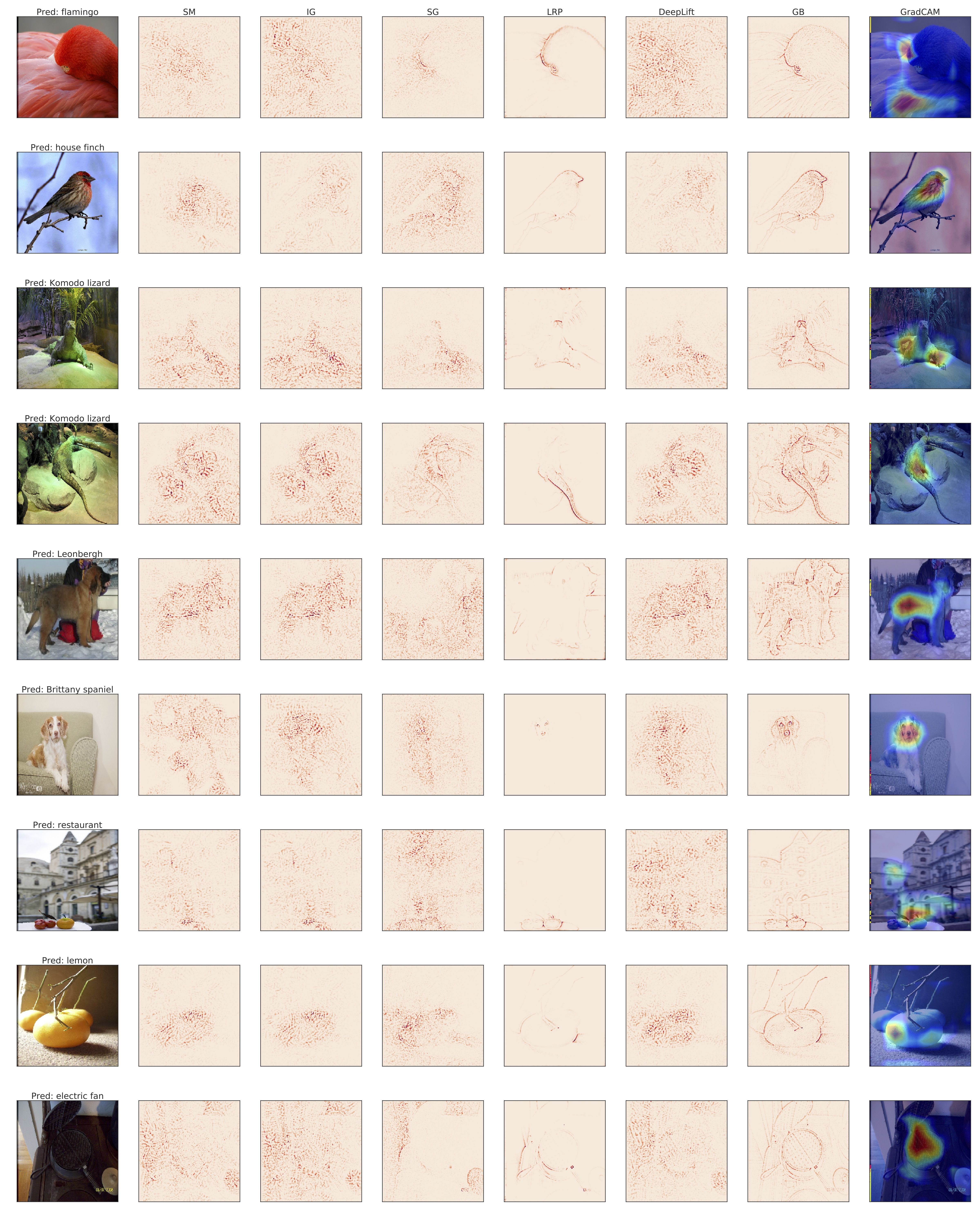}
         };
        
        % comment out the next line when done placing checkmarks
        % \draw[help lines,step=1.0in] (-3.0in,-4.0in) grid (3.0in,4.0in);
        
        % node[] at (-3in,-4in) {\LARGE $\checkmark$ };
         % \node[] at (3in,4in) {\LARGE $\checkmark$ };
         % red: TPN blue: TPS green: N-Ord yellow: S-Ord
         
         % TPN winner
         \node[] at (2.31in,3.15in) {\LARGE {\color{red} $\checkmark$ }}; %5
         \node[] at (1.53in,2.29in) {\LARGE {\color{red} $\checkmark$ }}; %4
         \node[] at (-0.79in,1.43in) {\LARGE {\color{red} $\checkmark$ }}; %1
         \node[] at (-1.52in,0.57in) {\LARGE {\color{red} $\checkmark$ }}; %0
         \node[] at (-0.01in,-0.29in) {\LARGE {\color{red} $\checkmark$ }}; %2
         \node[] at (-1.52in,-1.15in) {\LARGE {\color{red} $\checkmark$ }}; %0
         \node[] at (2.31in,-2.01in) {\LARGE {\color{red} $\checkmark$ }}; %5
         \node[] at (1.53in, -2.87in) {\LARGE {\color{red} $\checkmark$ }}; %4
         \node[] at (2.31in,-3.73in) {\LARGE {\color{red} $\checkmark$ }}; %5
         
        % TPS winner
        \node[] at (2.24in,3.15in) {\LARGE {\color{blue} $\checkmark$ }}; %5
        \node[] at (2.31in,2.29in) {\LARGE {\color{blue} $\checkmark$ }}; %5 
        \node[] at (3.08in,1.43in) {\LARGE {\color{blue} $\checkmark$ }}; %6
        \node[] at (0.75in,0.57in) {\LARGE {\color{blue} $\checkmark$ }}; %3
        \node[] at (2.31in,-0.29in) {\LARGE {\color{blue} $\checkmark$ }}; %5 
        \node[] at (-0.01in,-1.15in) {\LARGE {\color{blue} $\checkmark$ }}; %2
        \node[] at (2.31in,-2.01in) {\LARGE {\color{blue} $\checkmark$ }}; %5
        \node[] at (-0.01in,-2.87in) {\LARGE {\color{blue} $\checkmark$ }}; %2
        \node[] at (-0.01in,-3.73in) {\LARGE {\color{blue} $\checkmark$ }}; %2
        \end{tikzpicture}
        
    \end{center}
   \caption{More visualizations of different attribution methods. Red checks mark the winner of Total Proportionality for Necessity and blue checks mark the winner of Total Proportionality for Sufficiency}
   \label{fig:demonstration}
\end{figure*}
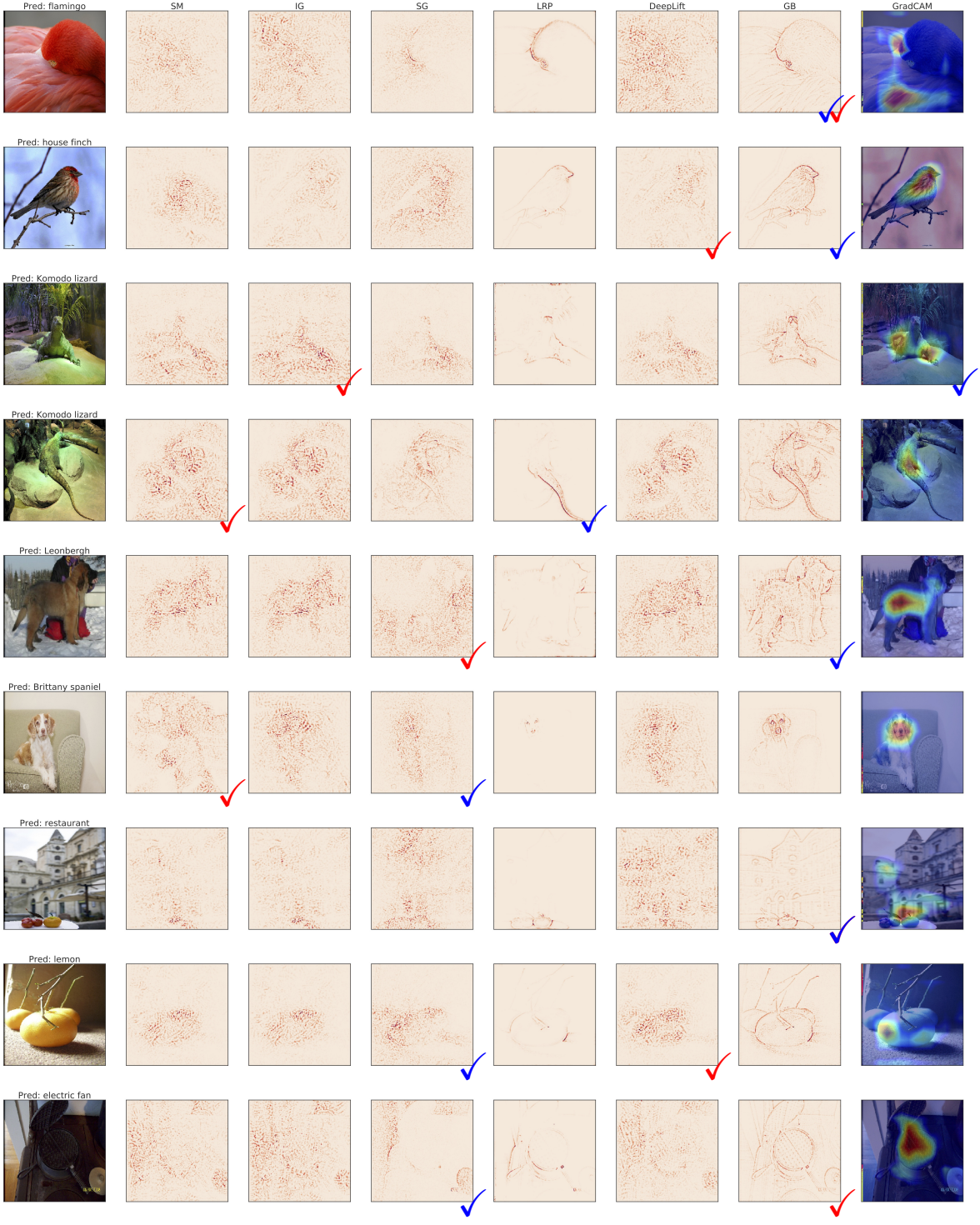
\end{document}